%% file: example_paper.tex
\documentclass{article}

\usepackage{microtype}
\usepackage{graphicx}
\usepackage{subfigure}
\usepackage{booktabs} % for professional tables

\usepackage{hyperref}
\usepackage{xspace}

\usepackage[accepted]{icml2024}

\usepackage{amsmath}
\usepackage{amssymb}
\usepackage{mathtools}
\usepackage{amsthm}
\newcommand{\alg}{\texttt{\textsc{BC-Max}}\xspace}
\usepackage[capitalize,noabbrev]{cleveref}

\theoremstyle{plain}
\newtheorem{theorem}{Theorem}[section]

\theoremstyle{definition}

\newtheorem{assumption}[theorem]{Assumption}

\theoremstyle{remark}

\usepackage[textsize=tiny]{todonotes}

\input{myheaders}

\icmltitlerunning{Offline Imitation Learning from Multiple Baselines}

\begin{document}

\twocolumn[
\icmltitle{Offline Imitation Learning from Multiple Baselines\\ with Applications to Compiler Optimization}

% It is OKAY to include author information, even for blind
% submissions: the style file will automatically remove it for you
% unless you've provided the [accepted] option to the icml2024
% package.

% List of affiliations: The first argument should be a (short)
% identifier you will use later to specify author affiliations
% Academic affiliations should list Department, University, City, Region, Country
% Industry affiliations should list Company, City, Region, Country

% You can specify symbols, otherwise they are numbered in order.
% Ideally, you should not use this facility. Affiliations will be numbered
% in order of appearance and this is the preferred way.
\icmlsetsymbol{equal}{*}

\begin{icmlauthorlist}
\icmlauthor{Teodor V. Marinov}{yyy}
\icmlauthor{Alekh Agarwal}{yyy}
\icmlauthor{Mircea Trofin}{comp}
% \icmlauthor{Firstname4 Lastname4}{sch}
% \icmlauthor{Firstname5 Lastname5}{yyy}
% \icmlauthor{Firstname6 Lastname6}{sch,yyy,comp}
% \icmlauthor{Firstname7 Lastname7}{comp}
%\icmlauthor{}{sch}
% \icmlauthor{Firstname8 Lastname8}{sch}
% \icmlauthor{Firstname8 Lastname8}{yyy,comp}
%\icmlauthor{}{sch}
%\icmlauthor{}{sch}
\end{icmlauthorlist}

\icmlaffiliation{yyy}{Google Research, USA}
\icmlaffiliation{comp}{Google, USA}
% \icmlaffiliation{sch}{School of ZZZ, Institute of WWW, Location, Country}

\icmlcorrespondingauthor{Teodor V. Marinov}{tvmarinov@google.com}
\icmlcorrespondingauthor{Alekh Agarwal}{alekhagarwal@google.com}
\icmlcorrespondingauthor{Mircea Trofin}{mtrofin@google.com}

% You may provide any keywords that you
% find helpful for describing your paper; these are used to populate
% the "keywords" metadata in the PDF but will not be shown in the document
\icmlkeywords{Machine Learning, ICML}

\vskip 0.3in
]

% this must go after the closing bracket ] following \twocolumn[ ...

% This command actually creates the footnote in the first column
% listing the affiliations and the copyright notice.
% The command takes one argument, which is text to display at the start of the footnote.
% The \icmlEqualContribution command is standard text for equal contribution.
% Remove it (just {}) if you do not need this facility.

%\printAffiliationsAndNotice{}  % leave blank if no need to mention equal contribution
\printAffiliationsAndNotice{} % otherwise use the standard text.

\begin{abstract}
This work studies a Reinforcement Learning (RL) problem in which we are given a set of trajectories collected with $K$ baseline policies. Each of these policies can be quite suboptimal in isolation, and have strong performance in complementary parts of the state space. The goal is to learn a policy which performs as well as the best combination of baselines on the entire state space. We propose a simple imitation learning based algorithm, show a sample complexity bound on its accuracy and prove that the the algorithm is minimax optimal by showing a matching lower bound. Further, we apply the algorithm in the setting of machine learning guided compiler optimization to learn policies for inlining programs with the objective of creating a small binary. We demonstrate that we can learn a policy that outperforms an initial policy learned via standard RL through a few iterations of our approach.
\end{abstract}

\section{Introduction}

When applying Reinforcement Learning (RL) to real-world applications, two key challenges often prove to be critical blockers to adoption. First is that the online interaction-then-update loop in conventional RL poses a significant engineering overhead in most large-scale systems, that are more naturally designed to take a static Machine Learning (ML) model as a dependency and update this model only periodically in an offline manner. Second is that RL algorithms typically begin tabula rasa, that is, they only leverage the information they glean about the task at hand through these online interactions. Typical scenarios, where the use of RL is often preceded by prior attempts using rule-based or supervised ML approaches, come with a treasure trove of valuable data about desirable and undesirable behaviors, ignoring which leads to undesirable sample complexity of learning from scratch for RL. More important, the previously tried decision making policies, even when individually suboptimal, provide a valuable source of insight into the plausibly good choices in many scenarios. In this work, we study the question of leveraging such prior policies and any data collected using them, without necessarily relying on online policy updates.

Given these shortcomings, offline RL~\citep{ernst2005tree,hester2018deep,kumar2020conservative,cheng2022adversarially}, where the agent just learns from a static dataset collected using some arbitrary policy, as well as hybrid protocols~\citep{song2022hybrid,haarnoja2018soft,silver2014deterministic} interpolating the fully online and offline settings have been proposed in the literature to take advantage of existing datasets, as well as to ease the requirement of completely online policy updates. In a different thread of work, the substantial literature on imitation learning~\citep{pomerleau1988alvinn,ross2011reduction,abbeel2004apprenticeship,ho2016generative} aims to leverage any existing policies that we seek to improve upon, along with the data collected using them. While imitation learning is studied in both online and offline settings, the particular scenario of having access to multiple policies of variable qualities that is of interest here, is only previously studied in an online setting~\citep{cheng2020policy,barreto2020fast}. 

In this paper, we formalize this question of having access to $K$ baseline policies $\pi_1,\ldots,\pi_K$, each of which can be quite suboptimal in isolation, and which we hope are strong in complementary parts of the state space. We further restrict ourselves to only receiving static datasets $D^i$ collected from each policy $\pi_i$, and seek to learn a policy which can combine the strengths of all the baseline policies. We are particularly interested in challenging settings where the underlying RL problem has a long horizon, and we only receive sparse trajectory-level feedback at the end of each trajectory. To motivate this setting, we consider a running example of optimizing the inlining decisions in a compiler. The horizon of the RL problem here corresponds to the number of callsites in the program or function being compiled., which can range from tens to tens of thousands. The reward of a trajectory is the size of the binary we obtain after compiling the entire function. The offline setting we study is extremely well motivated here, where each interaction with the RL environment involves the expensive operation of compiling and linking the program, and integrating this within the RL loop engenders a significant engineering overhead.. With this setup, our paper makes the following contributions:

\begin{itemize}
    \item A natural behavior cloning algorithm, \alg, that combines the multiple policies by executing each policy in every starting state in our dataset, and then imitating the trajectory of the policy with the highest reward in that state. We give an upper bound on the expected regret of the learned policy to the maximal reward obtained in each starting state by choosing the best baseline policy for that state.
    \item We complement our analysis with a lower bound showing that the result is unimprovable beyond polylogarithmic factors in our setting.
    \item We apply \alg to two different real-world datasets for the task of optimizing compiler inlining for binary size, and show that we outperform strong baselines in both the cases. In both cases we start with a single baseline policy, which is a prior model trained using online RL, which already has a strong performance on this task. We demonstrate the versatility of \alg by iteratively applying \alg on the initial expert, along with all prior policies trained using previous \alg iterations as the next set of baselines. We show that with a limited amount of interaction with the environment (to collect trajectories using each successive set of baselines), we obtain strong policies in a small number of iterations, creating a promising practical recipe for challenging real-world settings.
\end{itemize}

\section{Setting and Related Work}

We now define the problem setting formally, and then discuss some lines of prior work which are relevant to this setting.

\subsection{Problem setting}

\paragraph{Contextual MDP setting.} We consider a contextual Markov Decision Process (MDP) with state space $\mathcal{S}$ and action space $\mathcal{A}$. We denote the initial state distribution $D_1$ and the sampled context (initial state) is $x \sim D$.
Once the context is sampled, the transition kernel $\P_x$ is deterministic,
that is $\P_x(\cdot | s,a)$ is a point-mass distribution. The reward function is $r_x(s,a)$ and we assume deterministic rewards. We note that both the transition kernel and reward kernel are context-dependent. We will omit the context subscript from our notation whenever it does not introduce ambiguity.

We work in the finite-horizon setting and denote the horizon as $H$. The value function of a policy $\pi$ is 
\begin{align*}
    V_{\pi}(x) = \sum_{h=1}^H \E_{a_h\sim \pi(\cdot | s_h)} [r_x(s_h, a_h)],
\end{align*}
where $s_h$ is the the state at step $h$ s.t. $\P_x(s_h | s_{h-1}, a_{h-1}) = 1$ and $s_1 \equiv x$. Importantly, the policy class is such that the action distribution at state $s$ depends only on $s$ and not on the context, that is $\pi(\cdot |s, x) = \pi(\cdot |s)$.

\paragraph{Goal.} We assume that we are given a set of $K$ baselines policies $\{\pi_i\}_{k \in [K]}$ together with $n$ trajectories for each policy, which we denote by $\{\tau_{i,j}\}_{i\in[K], j\in[n]}$, where $x_j \sim D$ and a trajectory for policy $\pi$ consists of $\{(s_h, a_h)\}_{h\in[H]}$, with $a_h\sim \pi(\cdot | s_h)$. For an arbitrary policy $\pi$ and context $x$ we use $\tau_\pi(x)$ to denote the trajectory generated by following $\pi$ on context $x$.
We also assume that we see the total reward for each trajectory and policy, that is for all $i\in[K], j\in[n]$ we only observe
\begin{align*}
    r(\tau_{i, j}) = \sum_{(s_{j,h}, a_{j,h}) \in \tau_{i,j}}r(s_{j,h},a_{j,h}),
\end{align*} 
instead of observing a dense reward across all the time steps of the trajectory. We assume that the rewards are bounded and non-negative, that is $r(\tau) \in [0, B]$ for all trajectories $\tau$ and some constant $B$.

We have access to a policy class $\Pi$ and seek to find a find a policy in $\Pi$ which ideally competes with each of the baselines, and is able to combine their strengths. Let $V_i(x^1) = V_{\pi_i}(x^1)$ denote the expected cumulative reward of baseline $i$, conditioned on the context $x$. Then we seek to minimize the regret:
\begin{equation}
    \reg(\pi) := \E_{x\sim D}\left[\max_{i=1}^K V_i(x) - V_\pi(x)\right].
    \label{eq:regret}
\end{equation}
That is, we seek to compete with the best of the baselines for each individual context. 

\paragraph{Learning setup.} We assume access to the policy class $\Pi$, but do not assume any other function approximators, such as for modeling value functions. This is partly due to the fact that the typical training of value functions using Bellman backups is not feasible in our sparse-reward setting. Furthermore, typical actor-critic techniques make strong completeness and realizability assumptions on the value function class, which are not realistic with a restricted notion of state which we encounter in our motivating problem of compiler optimization. This necessitates the development of purely policy-based methods.

\subsection{Related work}

\paragraph{Vanilla behavior cloning} Behavior cloning~\cite{widrow1964pattern, pomerleau1988alvinn} refers to the approach of learning a policy that matches the mapping from states to actions observed in the data. This is typically solved as a classification problem for deterministic policies, or by maximizing the log-likelihood of the chosen actions in the observed states for stochastic policies. It is unclear how to apply vanilla behavior cloning in the presence of multiple baselines. We will present a natural formulation to behavior clone the best baseline policy per context in the following section.

\paragraph{Value-based improvement upon multiple baselines (MAMBA)} \citet{cheng2020policy} show how to simultaneously improve upon multiple baseline policies to compete with the best policy \emph{at each state} in the MDP, which is a significantly stronger notion that competing with the best baseline in each context only. However, this comes with two caveats. Their method requires value function estimation for the baselines and access to the MDP to execute trajectories under the learner's policy and/or baselines. \citet{barreto2020fast} also study a problem which involves improving over multiple baseline policies, which they title General Policy Improvement. The policy improvement step again requires value function evaluation. We do not assume such access to additional function approximators or the MDP in this work. 

\paragraph{Offline RL:} Without access to the MDP, a natural approach is to consider offline reinforcement learning, with the data collection policy being a mixture of the baselines $\pi_i$, say chosen uniformly. Given the recent results on offline RL to compete with any policy that is covered by the data distribution~\citep{kumar2020conservative, xie2021bellman, zhan2022offline}, we can expect a favorable bound on the regret~\eqref{eq:regret}, since all the baselines have a good coverage under the uniform data collection policy. However, existing offline RL methods with theoretical guarantees are typically based on function approximation, relying on actor-critic or $Q$-learning style approaches and on strong credit assignment using per timestep rewards rather than the aggregated reward of a trajectory. Applying these techniques using policy-based function approximation alone and with aggregated reward feedback is not feasible as we argue through a simple lower bound example.

\section{Algorithm and Regret Bound} 

We now describe our algorithm, \alg, and give an upper bound on the regret it incurs to the best per-context baseline. 

\paragraph{Algorithm.} We describe \alg in Algorithm~\ref{alg:bc}. The basic idea of the algorithm is quite simple. For each context $x_j$ in our dataset, we first choose the trajectory with the highest cumulative reward across all the baselines. Then we use a standard behavior cloning loss to mimic the choice of actions in this trajectory. For a context $x_j, j \in [n]$, we denote $i_j = \argmax_{i \in [K]} r(\tau_{i,j})$, and \alg tries to find a policy $\hat \pi \in \Pi$ that optimizes the following intuitive objective:
\begin{align}
\label{eq:alg_update}
    \hat \pi = \argmin_{\pi \in \Pi} \sum_{j=1}^n \sum_{(s_{j,h}, a_{j,h}) \in \tau_{i_j, j}}\mathbf{1}(\pi(s_{j,h}) \neq a_{j,h}).
\end{align}

\renewcommand{\algorithmicrequire}{\textbf{Input:}}
\renewcommand{\algorithmicensure}{\textbf{Output:}}
\begin{algorithm}
\caption{\alg for cloning best per-context baseline.}
\begin{algorithmic}
\REQUIRE Base policies $\{\pi_i\}_{i\in[K]}$ and policy class $\Pi$.
\ENSURE Policy $\hat\pi \in \Pi$.
\FOR{$j\in [n]$}
    \STATE Sample $x_j \sim D_1$, collect trajectories $\{\tau_{i,j}\}_{j\in[n]}$
    \STATE Compute highest reward policy $\pi_{i_j} = \argmin_{i\in[K]}\sum_{(s_{j,h}, a_{j,h}) \in \tau_{i,j}}r(s_{j,h},a_{j,h})$.
\ENDFOR
\STATE $\hat \pi = \argmin_{\pi \in \Pi} \sum_{j=1}^n \sum_{(s_{j,h}, a_{j,h}) \in \tau_{i_j,j}}\mathbf{1}(\pi(s_{j,h}) \neq a_{j,h})$.
\end{algorithmic}
\label{alg:bc}
\end{algorithm}

One natural question at this point might be that if there are two trajectories with very similar high rewards in a context, can it help to leverage this information rather than only picking the one with the higher reward and cloning it. This is indeed a shortcoming of \alg, and other behavior cloning style approaches. However, we note that we only have access to a trajectory-level reward, and hedging between two very different trajectories can create a very noisy learning setup for the algorithm. In situations where value-functions can be feasibly learned, such information is naturally modeled through the value function which assigns similar future rewards to similarly good actions, but we do not find a natural way for incorporating this information in our setup. 

\paragraph{Performance guarantee for \alg.} We now give a bound on the suboptimality of the policy learned by \alg, relative to the best per-context baseline, in terms of the rewards. The analysis mirrors the standard results for behavior cloning algorithms~\citep{ross2010efficient}. We begin with a realizability assumption which governs how well the best per-context baseline can be approximated using the learner's policy class $\Pi$.
\begin{assumption}
\label{assm:realizability}
Let $\tau^*(x) = \argmax_{\tau_{\pi_i}(x)} r(\tau_{\pi_i}(x))$ be the trajectory with maximum return over all policies $\pi_i, i\in[K]$. There exists $\pi^* \in \Pi$ such that 
\begin{align*}
    \P_{x\sim D}(\tau_{\pi^*}(x) \neq \tau^*(x)) \leq \epsilon.
\end{align*}
\end{assumption}

Here $\P_{x\sim D}(A)$ denotes the probability of an event $A$ under the distribution $D$, which we recall is the distribution over the contexts $x$. The assumption is natural as \alg cannot do a good job of approximating the best per-context baseline when no policy in the policy class has a small error in achieving this task. Note that the assumption does not take rewards into account as \alg only matches the actions of $\tau^*(x)$, and does not reason about the reward sub-optimality of other actions, as is common in behavior cloning setups. Indeed this assumption is unavoidable in our problem setting as we illustrate in the next section.
\begin{theorem}
Under Assumption~\ref{assm:realizability}, after collecting $n$ trajectories from each of the $K$ base policies, Algorithm~\ref{alg:bc} returns a policy $\hat \pi$ with regret at most
\begin{align*}
    \reg(\hat\pi) \leq O\bigg(\epsilon H + \frac{H^2\log(H|\Pi|/\delta)}{n}\bigg),
\end{align*}
with probability at least $1-\delta$.
\label{thm:ub}
\end{theorem}
\begin{proof}
Recall the definitions of $\tau^*(x)$ from Assumption~\ref{assm:realizability}, and let $\pi^* = \argmin_{\pi\in\Pi} \E_{x\sim D}(\sum_{h=1}^H \mathbf{1}(\pi(s(x)) \ne a(x)))$ where $(s_h(x), a_h(x))_{h=1}^H = \tau^*(x)$ form the best trajectory for $x$ among the baseline policies. Under Assumption~\ref{assm:realizability}, we know that $\E_{x\sim D}(\sum_{h=1}^H \mathbf{1}(\pi(s(x)) \ne a(x))) \leq \epsilon H$. Let us define
\begin{align*}
    \hat{A}(\pi) =& \sum_{j=1}^n \sum_{h=1}^H \mathbf{1}(\pi(s_{j,h}) \neq a_{j,h}),\\
    A(\pi) =& \E_x\left(\sum_{h=1}^H \mathbf{1}(\pi(s_h(x))\neq a_h(x))\right).
\end{align*}
Clearly we have that $\E[\hat{A}(\pi)] = A(\pi)$ for any fixed policy $\pi$, and $\hat{A}(\pi) = \sum_{j=1}^n Z_{j}(\pi)$ with $Z_{j}(\pi) = \sum_{h=1}^H \mathbf{1}(\pi(s_{j,h}) \neq a_{j,h}) \geq 0$. We note that the $Z_j$ are i.i.d., with $\E[Z_j(\pi)] = A(\pi)$ and $\E[Z_j(\pi)^2] \leq H\E[Z_j(\pi)] = H\,A(\pi)$. Then by Bernstein's inequality combined with a union bound over policies, we have with probability at least $1-\delta$, for all $\pi \in \Pi$:
\begin{align*}
    |\hat{A}(\pi) - nA(\pi)| \leq& \sqrt{nH\,A(\pi)\log(2|\Pi|/\delta)} + H\log(2|\Pi|/\delta)\\
    \leq& \frac{nA(\pi)}{2} + \frac{3}{2}H\log(2|\Pi|/\delta).
\end{align*}

Applying the inequality with $\pi = \hat{\pi}$ and $\pi = \pi^*$, we obtain

\begin{align*}
    A(\hat{\pi}) \leq& \frac{2}{n}\hat{A}(\hat{\pi}) + \frac{3H\log(2|\Pi|/\delta)}{n}\\
    \frac{1}{n}\hat{A}(\pi^*) \leq& \frac{3}{2}A(\pi^*) + \frac{3}{2}\frac{H\log(2|\Pi|/\delta)}{n}.
\end{align*}
Scaling the second inequality by 2 and adding them yields

\begin{align}
    A(\hat{\pi}) \leq 3A(\pi^*) + \frac{6H\log(2|\Pi|/\delta)}{n} \leq 3\epsilon + \frac{6H\log(2|\Pi|/\delta)}{n},
    \label{eq:0-1bound}
\end{align}
where the second inequality follows by Assumption~\ref{assm:realizability}. 

Now we note that for any policy $\pi$:
\begin{align*}
    \reg(\pi) &= \E_x [\max_i V_i(x) - V_{\pi}(x)] = \E_x[r(\tau(x)) - r(\tau_\pi(x))]\\
    &\leq \sum_{h=1}^H (H-h)\E_{x}[\mathbf{1}(\pi(s_h(x)) \neq a_h(x))|].
\end{align*}

Plugging the bound from Equation~\ref{eq:0-1bound} into the inequality above completes the proof.
\end{proof}

\paragraph{Implementation details}
In practice we can not directly compute $\hat \pi$ as defined in Equation~\ref{eq:alg_update}. Instead we solve a proxy to the optimization problem by replacing the indicator function $\mathbf{1}(\pi(s_{j,h}) \neq a_{j,h})$ by the cross-entropy loss. Let $y_{j,h} \in \{0,1\}^A$ be the indicator with only entry equal to $1$ the one which corresponds to the action $a_{j,h})$, and all other entries equal to $0$. Further, we assume that all $\pi \in \Pi$ are such that $\pi(s) \in \Delta^{A-1}$, that is each $\pi(s)$ represents a distribution over the actions that policy $\pi$ plays when in state $s$. We then use a first order method to minimize the loss
\begin{align*}
    \min_{\pi\in\Pi}\sum_{j=1}^n w_j\sum_{h=1}^H \sum_{a} -y_{j,h}(a)\log(\pi(S_{j,h})),
\end{align*}
where $w_j \in [0,1]$ are example weights which we find helpful in our practical implementation. We refer the reader to our experimental evaluation for mode details on how the weights are induced.

\section{Lower bounds}

In this section, we show a series of lower bounds which illustrate the necessity of various aspects of our guarantee in Theorem~\ref{thm:ub}. We start with the necessity of Assumption~\ref{assm:realizability}

\paragraph{Necessity of Assumption~\ref{assm:realizability}} Let us consider a contextual multi-armed bandit problem, meaning that we fix $H = 1$. For any $\epsilon$, we choose the context space $\cS = [M]$ for $M = \lceil\frac{1}{\epsilon}\rceil$, and choose $D$ to be the uniform distribution on $\cS$. We fix $\cA = \{a_1, a_2, a_3\}$, and $K=1$ with the data collection policy $\pi_1$ choosing $a=a_1$ for each context $x\in\cS$. We consider two possible environments, defined through rewards $r_1, r_2$. For $a_1$, we have $r_1(x, a_1) = r_2(x, a_2) = 1$. For the other two actions, we have $r_1(x,a_2) = 0$ and $r_1(x, a_3) = 1$, while the second environment has $r_2(x,a_2) = 1$ and $r_2(x, a_3) = 0$. We design a policy class $\pi$ with two policies $\{\pi_1, \pi_2\}$ such that $\pi_1(x) = \pi_2(x) = a_1$ for all $x\ne 1$ and $\pi_1(1) = a_2$m $\pi_2(1) = a_3$. Clearly this policy class satisfies Assumption~\ref{assm:realizability}. But it also contains an optimal policy for both the rewards $r_1, r_2$ with a regret equal to $0$. However, since the data contains no information about which one of $r_1$ or $r_2$ generated the data, the best we can do is to pick between $\pi_1$ and $\pi_2$ uniformly at random, and incur a regret of at least $0.5\epsilon$. This argument shows that we cannot replace the $0-1$ loss for measuring the accuracy of a policy in Assupmtion~\ref{assm:realizability}, with a more reward-aware quantity. It is also evident from the example that we cannot avoid incurring a regret of $\Omega(\epsilon)$.

\paragraph{Necessity of the comparator choice} As discussed earlier, we define regret relative to the best baseline per-context, but the broader literature on offline RL allows stronger benchmarks, such as the best policy covered by the data generating policies. To understand the difference between the two, we consider the case of $H=2$ and $K=2$, with $\cS = \{x\}$ being a singleton. Suppose we have two actions $\cA = \{a_1, a_2\}$ and the baselines choose the trajectories $\pi_1(x) = a_1$ and $\pi_2(x) = a_2$ at both $h=1,2$. There are two possible reward functions given by $r_1((a_1, a_1)) = r_1((a_2, a_2)) = 0.5$, $r_1((a_1, a_2)) = 1$, $r_1((a_2, a_1)) = 0$ and $r_2(\tau) = 1-r_1(\tau)$ for all trajectories $\tau$. Now we observe that in the sense of the coverage studied in the offline RL literature, all four trajectories are covered by the dataset, since we get to observe both the actions at both the steps in the episode. However, since the data contains no useful information to distinguish between $r_1$ and $r_2$, no learning method can pick a covered policy which is better than our benchmark of best baseline per-context. This challenge arises in our scenario as we only observe the aggregate reward over a trajectory, which makes per-step credit assignment used in standard offline RL methods through the use of Bellman errors infeasible.

\paragraph{Tightness of horizon factor}
The tightness of the horizon factor follows from a simple reduction to Theorem 6.1 in \citet{rajaraman2020toward}, who show that in a finite horizon episodic MDP, there is no algorithm which only observes $n$ optimal policy trajectories and returns a policy with regret better than $\Omega\left(\frac{SH^2}{n}\right)$. The MDP contructed by \citet{rajaraman2020toward} for the lower bound has a non-deterministic transition kernel $\P$. Since we are assume that the transitions are deterministic, we instead use the randomness in sampling contexts, $x\sim D$ to simulate $\P$. Concretely, let $\xi_1,\ldots, \xi_H$ be the random bits sampled at the $H$ steps in a fixed episode from the construction in \citet{rajaraman2020toward}. We define $x = (\xi_1,\ldots,\xi_H)$ and set $\P_x (s_{h+1} | s_h, a_h) = \P(s_{h+1} | s_h, a_h, \xi_h)$ to be the trasition taken for the realization of $\xi_h$. The size of the state space of contextual MDP with the above transition kernel can be set to $S = \Theta(\log_2(|\Pi|)$ and the action space to $\cA = \{a_1, a_2\}$, so that each policy in $\Pi$ is encoded by how it acts on the state space. This construction directly leads to the following lower bound.

\begin{theorem}
For any number of samples $n$ there exists a family of contextual MDPs with disribution over contexts given by $D_1$, such that the policy $\hat \pi = \texttt{A}(\{\tau_{\pi^*}(x_i)\}_{i=1}^n)$ returned by any algorithm $\texttt{A}$ satisfies
\begin{align*}
    \E_{x}[V_{\pi^*}(x)] - \E_x[V_{\hat\pi}(x)] \geq \min\left\{H, \frac{\log_2(|\Pi|)H^2}{n}\right\}.
\end{align*}
\end{theorem}

\section{Case study: Optimizing a compiler's inlining policy}
\subsection{The inlining for size problem}
In short the inlining problem which we study in our experiments consists of deciding to inline or not to inline a callsite in a program with the goal of minimizing the size of the final program binary. We omit most compilation details and just give a brief overview which should be sufficient for understanding the problem from a RL perspective. In our specific scenario, compilation is split into a frontend (fe) and a backend (be). 
In our setup the frontend consists of translating the program into an Intermediate Representation (IR), doing some frontend optimizations,
then a (thin) link step follows, which re-organizes functions in the various modules to improve inlining opportunities. For more details on the linking step see~\citet{johnson2017thinlto}.
The backend compilation follows the thin link step and is applied on the updated modules. It consists of further optimizations, including inlining decisions, final linking and lowering the IR to machine code, e.g., x86, ARM, etc. The IRs with which our RL algorithms work with are post frontend linking and pre backend optimization, that is we work only on backend optimizations.
We note that a program is made up of multiple \emph{modules}.
In the fe, a module corresponds to a single C/C++ source file, after all the preprocessor directives have been applied. In the be, the module would consist of a mix of IRs from different fe modules.
The inlining decisions will be taken at callsites in the IRs of each module, where the callee is also in the same module.
We note that each module is ultimately compiled to a machine code-specific binary, with its own size, that will still need to be linked into the final executable.
Hence we treat each module as a context $x$ in our contextual MDP setting, and the value $V_i(x)$ of the baseline policy $\pi_i$ is the size of the binary we get for module $x$, when we make inlining decisions for the module according to $\pi_i$. 
\begin{equation}
\label{eq:compiler_workflow}
    \begin{aligned}
        &\texttt{Program} \to \texttt{IRs} \to \texttt{\textbf{fe optimizations}}\\
        \to &\texttt{ThinLinking\footnotemark} \xrightarrow{\texttt{\textbf{Collecting IRs}}} \texttt{\textbf{be optimization}}\\
        \to &\texttt{Final linking} \to \texttt{x86}
    \end{aligned}
\end{equation}
\footnotetext{See~\citet{johnson2017thinlto}}
In Equation~\ref{eq:compiler_workflow} we outline the compilation process, together with the step at which we collect IRs from the respective modules to be used in our RL algorithms. It is important to note that the learned RL policy will make inlining decisions both at the fe optimization and be optimization parts in Equation~\ref{eq:compiler_workflow}, however, the IRs for training are only collected after the fe optimization step.

The contextual MDP setting can now be tied together with the compilation process as follows. Each context is a module as mentioned above, with state-space defined by its IR. Each state corresponds to a callsite in the IR of the module $x$, and the action set is $\{\texttt{inline}, \texttt{ don't inline}\}$ or $\{1, 0\}$ respectively. Each $\pi_i$ is some base inlining policy and $V_i(x)$ is defined by the size of the compiled stand-alone module $x$. It is important to note that there is a mismatch between trying to maximize $\E_{x\sim D}[V_{\pi}(x)]$ and the overall goal of minimizing the binary size, as it is not necessarily true that the sum of module sizes equals the size of the binary. In fact, part of the post-inlining \texttt{be} and linker optimizations may introduce a significant distribution shift between the sum of module sizes and the size of the final binary. In our experiments, we try to minimize this distribution shift by turning off certain optimizations. For more details on the compilation pipeline we refer to \citet{trofin2021mlgo}.

We note that the entire process is fully deterministic, as we assumed in our theoretical setup, since the compiler is a deterministic program. 

\subsection{Dataset collection}
\label{sec:data_collection}
We train and evaluate on two sets of binaries. In the first experiment we train on a proprietary search binary and evaluate the model on a different proprietary set of targets that are part of a cloud systems infrastructure. These targets need to be installed on a fixed size partition of each cloud machine and hence are size-constrained.
In the second experiment we train and evaluate on the Chrome binary on Android. Training proceeds in two separate steps, which we repeat over several iterations. The two steps can be summarized as follows, first we collect a training dataset which consists of trajectories with smallest size over all base policies available at the current iteration. Next, we train a new base model using the objective defined in Equation~\ref{eq:alg_update}. This conceptually applies Algorithm~\ref{alg:bc} repeatedly, where the set of baseline policies is updated at each iteration to include the new policy obtained from the previous iteration. We now describe each step carefully.
\paragraph{Training dataset collection.}
The dataset collection begins by creating a corpus of IRs of modules which make up the final binary. The corpus creation follows the work of \citet{trofin2021mlgo, grossman2023compile} and uses tools for extracting the corpus are available on GitHub\footnote{A detailed example can be found at \url{https://github.com/google/ml-compiler-opt/blob/main/docs/inlining-demo/demo.md}}. The corpus is created at the beginning of training and remains the same throughout every iteration. Training begins under the assumption that there exists at least one base policy. In the first iteration a training dataset is collected from this initial base policy $\pi_1$, next, $\pi_1$ is behavior cloned by solving the optimization problem in Equation~\ref{eq:alg_update}. We specify how the weights for the objective are computed in the following sections as they are different for the different targets. Let $\hat \pi_2$ denote the resulting policy after solving Equation~\ref{eq:alg_update}. This policy is non-deterministic and so we construct the base policy $\pi_2$ by setting $\pi_2(s) = \argmax_{a \in \{0,1\}} \hat \pi_2(s,a)$, that $\pi_2$ always plays the most likely action according to $\hat \pi_2$. This concludes the first iteration.
More generally, if we have a larger initial set of baseline policies than just a singleton$\{\pi_1\}$, the iterations proceed similarly. However, instead of just using $\pi_1$, we use the full set $\{\pi_i\}_{i\in[K]}$ of baseline policies at every iteration at the first iteration.

Proceeding this way, at the $t$-th iteration the set of base policies is taken as a subset of $\{\pi_1,\ldots,\pi_{t-1}\}$ which always contains $\pi_1$ (or the larger set of all initial baselines). Then we again invoke \alg with these baseline policies, and obtain a new randomized policy $\hat \pi_t$, and we refer to $\pi_t$ as the corresponding deterministic greedy policy.
When collecting a new training dataset we not only collect trajectories with the chosen subset of base policies but we also may force exploration by using $\hat \pi_{t-1}$ in the way discussed next.

\paragraph{Exploration in training dataset collection.}
For a fixed module $x$ and a policy $\pi$, we choose a ceiling on the number of exploration steps as a hyper parameter, which is a fraction of the length of the trajectory $|\tau_{\pi_1}(x)|$. The call-sites at which exploration occurs are selected as follows. The first exploration call-site is selected as $\tilde h = \argmin_{h}\{|\hat \pi(s_h)(0) - \hat \pi(s_h)(1)|\}_{S_h \in \tau_{\pi}(x)}$, as the call-site where the exploration policy $\hat \pi$ is the least confident about the action to choose. The exploration step is then played at $s_{\tilde h}$ by taking the action $1 - \pi(s_{\tilde h})$ (recall that $\pi(s) \in \{0,1\}$), and the remaining steps in the trajectory are completed by playing according to $\pi$. Let $\hat \tau$ denote the trajectory from the last round of exploration. In the following exploration round the exploration step is selected as the step $h$ at which the gap, $|\hat \pi(s_h)(0) - \hat \pi(s_h)(0)|$, is smallest among all $h > \tilde h$, where $\tilde h$ is the exploration step in the previous round. Once the maximum number of exploration rounds is reached or the exploration step reaches the end of the trajectory, we return the trajectory which results in the smallest module size among all explored trajectories.
Pseudo-code is presented in Algorithm~\ref{alg:explore_module}.
The exploration strategy is governed by $\hat \pi_{t-1}$, however, it can be updated by using the non-deterministic policy $\hat \pi$ which induces $\pi$. We leave such approaches as future work, as we have already observed significant benefit of using only $\hat \pi_{t-1}$ as the exploration policy.

\begin{algorithm}
\caption{Explore module}
\begin{algorithmic}
\REQUIRE Base policy $\pi$, exploration policy $\hat\pi$, module $x$, maximum exploration steps $T$.
\ENSURE Compilation trajectory $\tau_{\pi}(x)$ with reward $r_{\pi,x}$.
\STATE Compute vanilla trajectory $\tau_\pi(x)$ by compiling with $\pi$ and receive reward $r^1_{\pi,x}$
\STATE $t=1$
\STATE $\hat \tau_1 = \tau_\pi(x)$
\STATE $\tilde h_1 = \argmin_{h}\{|\hat \pi(S_h)(0) - \hat \pi(S_h)(1)|\}_{S_h \in \hat \tau_1}$
\WHILE{$t\leq T$}
    \STATE Replay $\hat \tau_t$ until $\tilde h_t$
    \STATE Play $1 - \pi(S_{\tilde h_t})$ at $\tilde h_t$
    \STATE Complete trajectory $\hat \tau^{t+1}$ by playing $\pi$
    \STATE Receive reward $r^{t+1}_{\pi,x}$
    \IF{$\tilde h_t < |\hat \tau_{t+1}|$}
    \STATE $\tilde h_{t+1} = \argmin_{h > \tilde h_t}\{|\hat \pi(S_h)(0) - \hat \pi(S_h)(1)|\}_{S_h \in \hat \tau^{t+1}}$
    \ELSE
    \STATE \textbf{break}
    \ENDIF
\ENDWHILE
\STATE $t^* = \argmax_{t} r^t_{\pi, x}$, $r_{\pi, x} = r^{t^*}_{\pi, x}$, $\tau_{\pi}(x) = \hat\tau_{t^*}$
\end{algorithmic}
\label{alg:explore_module}
\end{algorithm}

\paragraph{Online versus offline learning.} Our theoretical setup frames the problem in an offline learning scenario, yet Algorithm~\ref{alg:explore_module} and the iterative procedure do rely on our ability to interact with the environment in an adaptive manner. Note, however, that the modality of interaction used in our approach is quite different, and significantly more practical than full-fledged online RL. Each round of policy learning, which happens using \alg, is fully offline. This process, which involves a large number ($10^5-10^6$) stochastic gradient steps, happens without any interaction with the environment, and is where the bulk of the learning happens. Subsequently, we form a new data collection policy for the next iteration, and this policy is applied to collect one trajectory per module. The data collection process does not involve any policy updates, and hence is massively parallelizable with no interlocking bottlenecks with the learning process. In online RL, on the other hand, data collection and policy updating go hand-in-hand, which typically requires significantly more complex architecture~\citep{mnih2016asynchronous} to scale to domains where data collection is expensive. Our approach, on the other hand, simply requires interleaving standard supervised learning and batch data collection, which is quite desirable especially in the compiler application, where the ML training happens on GPUs, while the compilation happens on CPU machines. 

\subsection{Search application targets}
\label{sec:ker_res}
Similarly to \citet{trofin2021mlgo} we collect a corpus for training purposes from a search application binary with approximately $30000$ modules. The initial base policy is an RL model trained using an Evolutionary Strategy (ES\footnote{The policy can be found here: \url{https://github.com/google/ml-compiler-opt/releases/tag/inlining-Oz-v1.1}}) as in \citet{trofin2021mlgo}. After collecting a training dataset with the ES policy we noticed that the distribution of sizes of modules is fairly non-uniform, with few modules having very large sizes or very small sizes and majority of modules being somewhere in-between. Because we expect that the actions of the behavior cloning policy taken on larger size modules are more important for size saving we upweight the actions in such trajectories. The weights used in training are computed as follows. Let $\texttt{size}(x, \pi_1)$ denote the size of module $x$ from the collected trajectory under policy $\pi_1$ (or in the case of multiple baseline policy under the best baseline policy). The modules are partitioned into buckets according to their sizes where the limits of the buckets are taken to be on exponentially scaling grid, that is the first bucket contains all modules with size $\texttt{size}(x, \pi_1)\in [0, 2^0)$, the second bucket all modules such that $\texttt{size}(x, \pi_1) \in [2^0, 2^1)$ etc., up to the final bucket with size $[2^{M-1}, 2^{M})$. Let $b_m = \{x : \texttt{size}(x, \pi_1) \in [2^{m-1}, m)\}$ denote the $m$-th bucket and let $m(x)$ be the $m$ for which $x \in b_{m(x)}$. The weight $w_x$ for module $x$ is computed as follows.
\begin{align*}
    w_x = \frac{\max_m |b_m|}{|b_{m(x)}|}.
\end{align*}

\begin{algorithm}
\caption{\alg for cloning best per-context baseline with exploration}
\begin{algorithmic}
\REQUIRE Base policy $\pi_1$ and policy class $\Pi$. Max exploration steps $T$. Max number of iterations $N$. 
\ENSURE Policy $\hat\pi \in \Pi$.
\STATE $l = 1$
\WHILE{$l \leq N$}
\FOR{$j\in [n]$}
    \STATE Sample $x_j \sim D_1$
    \FOR{$i \in \{\pi_s\}_{s \leq l}$}
        \STATE $r_{i,j}, \tau_{i,j} = \texttt{Algorithm~\ref{alg:explore_module}}(\pi_i, \hat \pi_\ell, x_j, T)$
    \ENDFOR
    \STATE Compute highest reward policy $\pi_{i_j} = \argmin_{i\in[K]}\sum_{(s_{j,h}, a_{j,h}) \in \tau_{i,j}}r_{i,j}(s_{j,h},a_{j,h})$.
\ENDFOR
\STATE Compute module weights $\{w_j\}_{j\in [n]}$ (See Sect.~\ref{sec:ker_res},\ref{sec:android}).
\STATE $$\hat \pi_l = \argmin_{\pi \in \Pi} -w_j\sum_{j=1}^n \sum_{(s_{j,h}, a_{j,h}) \in \tau_{i_j,j}}a_{j,h}\log(\pi(a_{j,h}|s_{j,h}))$$
\STATE $\pi_l(s) = \argmax_{a}\hat \pi_l(a|s), \forall s\in S$
\ENDWHILE
\end{algorithmic}
\label{alg:bce}
\end{algorithm}

We train two sets of policies, one set is trained without exploration and is precisely in line with Algorithm~\ref{alg:bc}. For full pseudo-code, which includes the exploration step, we refer the reader to Algorithm~\ref{alg:bce}.
The second set is trained with exploration as described in Section~\ref{sec:data_collection}.
\begin{figure}[ht]
    \centering
    \includegraphics[width=0.5\textwidth]{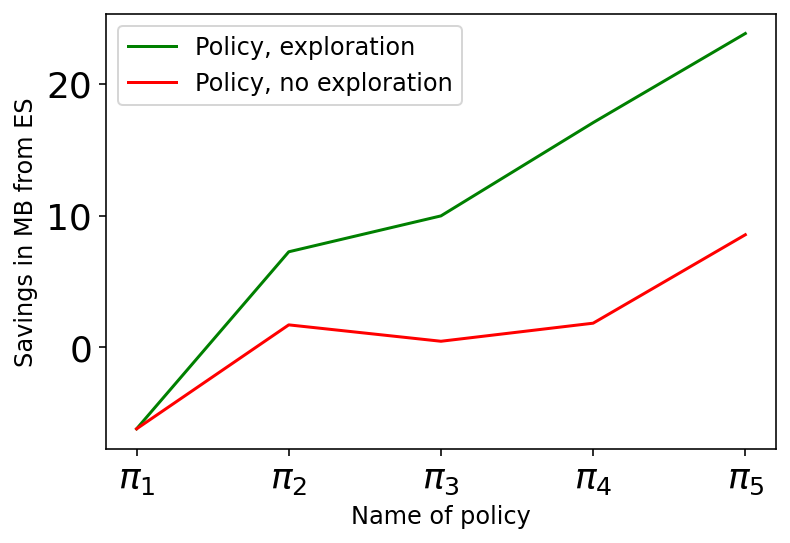}
    \caption{Savings in MB from ES on training binary}
    \label{fig:training_plot}
\end{figure}
In Figure~\ref{fig:training_plot} we show savings of the trained policies to $\pi_1$, which is the ES policy, on the search binary from which the training dataset is collected. In Figure~\ref{fig:test_plot} we show the savings on a \emph{different test binary}. On the $x$-axis of the figures we show the size savings of the policy $\pi_i$ learned at each iteration $i$, with and without exploration respectively, where $bc_0$ is the behavior cloned policy from ES.
Both figures demonstrate the success of our approach in improving significantly beyond the initial baseline, as well as the benefits from multiple iterations of the process. Furthermore, the gap between the lines with and without exploration highlights the benefits of the added exploration. 

We note that the compilation for both the training and test binaries is carried out in the following way to minimize the distribution shift -- the fe optimizations are carried out by ES, while the be optimizations are carried out by the trained policies. If we were to use the trained policies in both fe and be, this might lead to significant distribution shift, as Algorithm~\ref{alg:bc} works only on trajectories collected after the fe optimizations for which ES is always used. That is, if any of the trained policies, $bc_i$, act very differently on the fe, compared to ES, the resulting IRs before the be optimization might be completely different from the training set IRs, and hence the trained policy might take very sub-optimal actions. 
\begin{figure}[ht]
    \centering
    \includegraphics[width=0.5\textwidth]{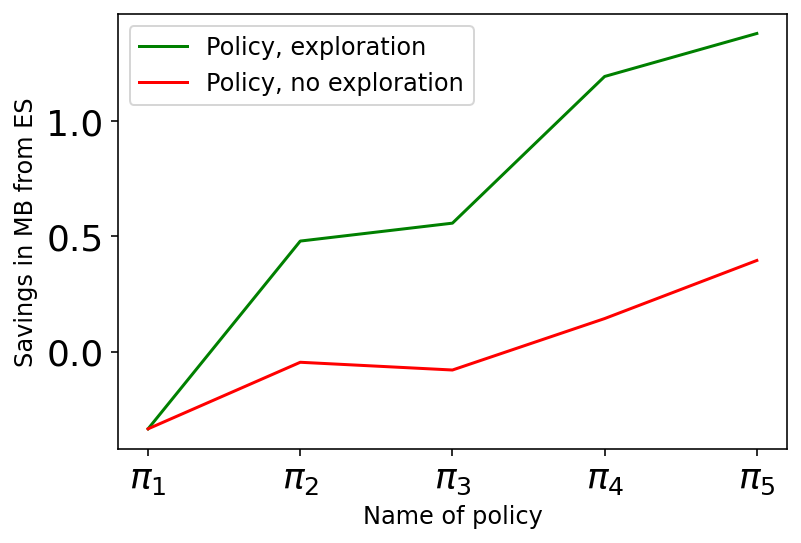}
    \caption{Savings in MB from ES on test binary}
    \label{fig:test_plot}
\end{figure}

\subsection{Chrome on Android}
\label{sec:android}
In our second set of experiments we train an RL policy for Chrome on Android.
The training and test binaries are the same in this case. The base policy with which we start is an RL policy trained using Proximal Policy Optimization (PPO\footnote{The policy can be found here: \url{https://commondatastorage.googleapis.com/chromium-browser-clang/tools/mlgo_model2.tgz}}) \citep{schulman2017proximal} as done in \citet{trofin2021mlgo}.
There are two differences in training from Section~\ref{sec:ker_res}. First, we focus only on the setting where we do exploration. The second difference in training is how the weights for the objective in Equation~\ref{eq:alg_update} are formed. The approach for computing the weights used here is inspired by the fact that we want to improve on PPO in each module and not just on the sum of module sizes. That is we want to maximize the size savings over the worst case module in our dataset. The following approach is natural when such max-min guarantees are desired.

Reusing notation from Section~\ref{sec:ker_res} we let
\begin{align*}
    p_x^1 &= \frac{|b_{m(x)}|}{\sum_{m} |b_m|}\\
    w_x^1 &= \frac{\max_m p_x^1}{p_{m(x)}^1},
\end{align*}
be the weights in the first iteration of training. In following iterations the weights are set as $w_x^t = \frac{\max_m p_x^t}{p_{m(x)}^t}$, where $p^t$ is update using the Hedge algorithm \citep{littlestone1994weighted}. The update uses the sum of sizes in each bucket as losses, normalized by the $\ell$-infinity norm, that is
\begin{align*}
    \tilde L^t_m &= \sum_{x \in b_m} \texttt{size}(x, \pi_t)\\
    L^t_m &= \frac{\tilde L^t_m}{\|L^t\|_\infty},
\end{align*}
where $L_m^t$ denotes the $m$-th coordinate of the loss vector $L^t$.
The Hedge update is then
\begin{align*}
    \tilde p^{t+1}_m &= p^{t}_m\exp(-\eta L^t_m)\\
    p^{t+1}_m &= \frac{\tilde p^{t}_m}{\sum_{m} \tilde p^t_m}.
\end{align*}
\begin{figure}
    \centering
    \includegraphics[width=0.45\textwidth]{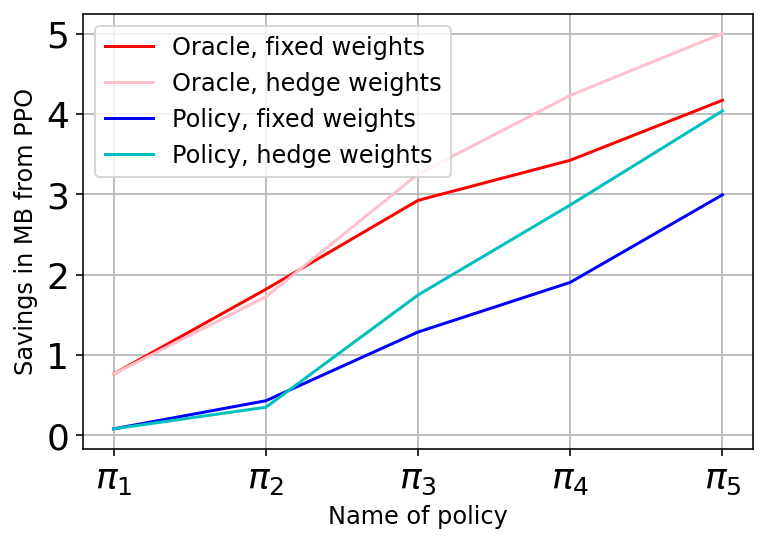}
    \caption{Savings in MBs from PPO on sum of module sizes}
    \label{fig:modules}
\end{figure}
\begin{figure}
    \centering
    \includegraphics[width=0.45\textwidth]{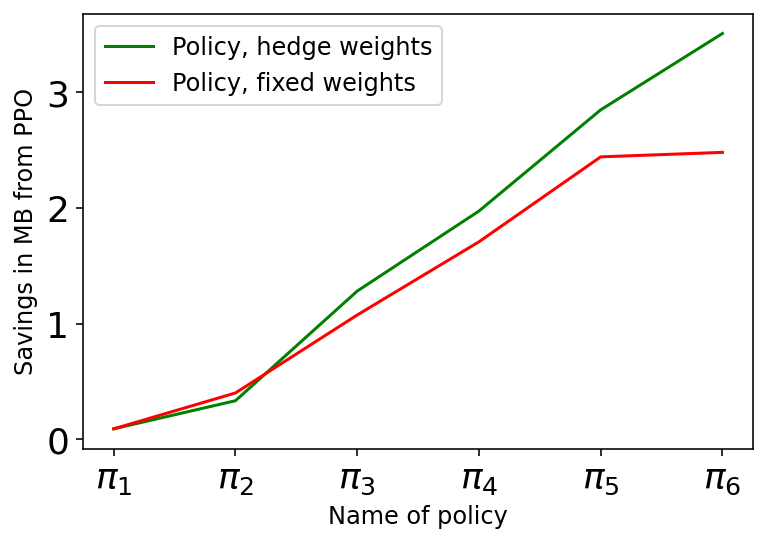}
    \caption{Savings in MBs from PPO on binary size}
    \label{fig:binary}
\end{figure}
In Figures~\ref{fig:modules} and \ref{fig:binary}, we plot the savings of our learned policies across iterations, relative to the initial PPO policy, measured in two different ways. For Fig~\ref{fig:modules}, we simply add up the sizes of the binaries produced by compiling each module in our training dataset. This is a clean metric, as the distribution shift between training and evaluation is small, and no artifacts from linker or post-inlining \texttt{be} optimizations are introduced in the evaluation. As we see, we improve rapidly beyond the PPO policy with the iterative applications of \alg. Note that even the sum of module sizes suffers from the typical distribution shift between online and offline RL, since the data used from behavior cloning is collected using a different policy than the one we apply in evaluation. For the sum of module sizes metric, we can study the effect of this distribution shift rather carefully by also compiling with an \emph{oracle} policy, which simply chooses the best baseline policy for each module, which is the target for training in \alg. This oracle, shown in red in Figure~\ref{fig:modules} naturally provides larger gains relative to PPO than our learned policy as expected, but the gap reduces through the iterations of our process, indicating that the policies tend to stabilize through iterations, and the training data for later applications of \alg is closer to on-policy data. We note that the oracle changes between different instantiations of our weights. This is because the $i$-th trained policy $\pi_i$ depends on the choice of weights and so the oracle after the $i$-th iteration which chooses the best among $\{\pi_i\}_{\ell=1}^i$ also depends on the choice of weights.
We also note that the gap between the learned and oracle policy's performance is smaller when we use the Hedge weights, and that the weighted version has a bigger gain relative to PPO, showing the efficacy of this approach.

Finally, in Figure~\ref{fig:binary} we present the savings in size of the Chrome on Android binary, which is the actual yardstick. Here we cannot easily evaluate the size of the oracle, so we only compare our policies to PPO, and again observe impressive gains, with the Hedge-weighted variant doing better. The binary size when compiled with the PPO policy is approximately $213.32$ MB.

\section*{Acknowledgements}
We would like to thank Ziteng Sun for great discussions on the lower bound.

\bibliography{example_paper}
\bibliographystyle{icml2024}
\end{document}

%% file: myheaders.tex
\usepackage{amsmath,amsfonts,amssymb, amsthm}
\usepackage{mathrsfs}
\usepackage{xcolor}
\usepackage{algorithm, algorithmic}
\usepackage{xspace}
\usepackage{natbib}
\usepackage{bm}
\usepackage{enumitem}
\theoremstyle{definition}

\newcommand{\cA}{{\mathcal A}}

\renewcommand{\P}{\ensuremath{\mathbb{P}}}

\newcommand{\cS}{\ensuremath{\mathcal{S}}}

\DeclareMathOperator{\E}{{\mathbb E}}

\DeclareMathOperator*{\argmax}{\text{argmax}}
\DeclareMathOperator*{\argmin}{\text{argmin}}

\usepackage{eqparbox}
\renewcommand{\algorithmiccomment}[1]{\hspace*{\fill}\mbox{\textcolor{blue}{$\triangleright$ #1}}}

\newcommand{\reg}{\ensuremath{\text{Reg}}}